\documentclass[11pt]{article}
\usepackage{amsmath}
\usepackage{amsfonts}
\usepackage[preprint]{acl}

\usepackage{times}
\usepackage{latexsym}

\usepackage[T1]{fontenc}

\usepackage[utf8]{inputenc}

\usepackage{microtype}

\usepackage{inconsolata}

\usepackage{graphicx}
\usepackage{booktabs}
\usepackage{siunitx}
\usepackage{multirow}

\usepackage{amsthm}  
\theoremstyle{definition}

\usepackage{thmtools,thm-restate}

\newtheorem{assumption}{Assumption}

\usepackage{pifont}      
\usepackage{xcolor}      
\usepackage{makecell}

\newcommand{\cmark}{\textcolor{teal!65!black}{\ding{51}}}  
\newcommand{\xmark}{\textcolor{red!55!black}{\ding{55}}}   

\title{
Models Know Their Shortcuts: Deployment-Time Shortcut Mitigation}

\author{Jiayi Li, Shijie Tang, Gün Kaynar, Shiyi Du, Carl Kingsford\thanks{corresponding author.} \\
        Ray and Stephanie Lane Computational Biology Department,\\ School of Computer Science, Carnegie Mellon University, Pittsburgh, PA \\
        \texttt{\{lijiayi,shijiet,gkaynar,shiyid,carlk\}@cs.cmu.edu}
        }

\begin{document}
\maketitle

\begin{abstract}
Pretrained text encoders are prone to \textit{shortcut learning}, relying on token-label correlations that fail once the distribution shifts in deployment. 
Existing shortcut mitigation methods mainly operate at training time and assume access to training data, training dynamics, or shortcut annotations, which are hardly available during deployment, where only the converged model remains.
We show that this model alone suffices to mitigate shortcuts during deployment: a biased model internalizes a signal of its learned shortcuts that can be captured via unsupervised gradient-based attribution. 
We further prove that deployment-time mitigation is information-theoretically upper-bounded by training-time mitigation. Nevertheless, exploiting this gradient signal, our proposed unsupervised deployment-time shortcut mitigation framework for pretrained text encoders, \textsc{Shortcut Guardrail}, recovers substantial performance under shortcut distribution shift, matching or outperforming training-time baselines across sentiment classification, toxicity detection, and natural language inference.
\end{abstract}

\section{Introduction}
Pretrained language models (LMs), especially encoder models such as BERT~\citep{devlin2019bert}, form the foundation of a wide range of natural language processing (NLP) tasks~\citep{wang2019glue}, including sentiment classification~\citep{pang2002thumbs,socher2013recursive}, natural language inference~\citep{bowman2015large}, and toxicity filtering~\citep{borkan2019nuanced}. Despite their strong empirical performance, these models can be easily biased through \textit{shortcut learning} ~\citep{geirhos2020shortcut, du2022shortcut_survey}, where predictions are driven by superficial patterns rather than task-relevant semantics. 
For instance, a sentiment classifier may associate the token ``service'' with positive sentiment because most training reviews containing it happen to be positive (Figure~\ref{fig:example}). 
Over-reliance on such shortcuts 
can lead to brittle generalization when the underlying token-label correlations no longer hold at deployment time~\citep{tu2020empirical}.

\begin{figure}[htbp]
    \centering
    \includegraphics[width=1\linewidth]{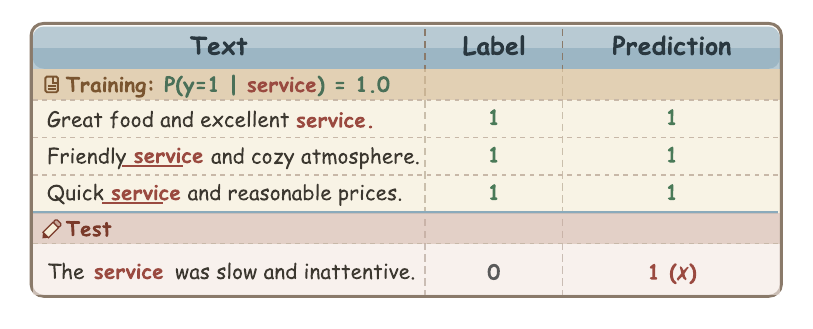}
    \caption{Example of shortcut learning in sentiment classification. The classifier associates the token ``service'' with \textit{positive} sentiment, leading to misclassification when it appears in a negative review at test time.}
    \label{fig:example}
\end{figure}

Prior research on detecting or mitigating shortcut learning mainly operates at the training stage. 
Detection methods diagnose shortcuts from training set statistics and training dynamics~\citep{du2021towards}; mitigation methods remove shortcuts during training, requiring shortcut annotations~\citep{sagawa2020distributionally}, regularization targets~\citep{chew2024nfl}, or access to the training data for retraining and reweighting~\citep{liu2021just_train_twice, clark2019dont_take_easy, kirichenko2023last_layer}.
Yet these assumptions fail where shortcuts surface as a problem during deployment: only a converged model (with open weights) and the test data are accessible.
This raises an open question: Can shortcuts be detected and mitigated using only this converged model and test data?

We find that this is possible. Models ``know'' their shortcuts since they leave a signal that can be uncovered through unsupervised, gradient-based attribution, requiring no prior knowledge of training data or shortcut annotations.
Specifically, the identified high-attribution tokens effectively capture the shortcut tokens for biased models, yielding $>90\%$ recall under a strong shortcut signal (Section~\ref{sec:tokens in important tokens}).
However, deployment-time mitigation is fundamentally limited with the converged model acting as an information bottleneck between training and deployment.
We formally prove that the performance of deployment-time mitigation is information-theoretically upper-bounded by training-time mitigation. 

Nevertheless, we find that the aforementioned signal is informative enough to bridge a large part of this gap. 
We propose \textsc{Shortcut Guardrail}, a lightweight, label-free mitigation framework for pretrained text encoders that exploits this gradient signal to identify candidate shortcut tokens for each test input, then trains a Low-Rank Adaptation (LoRA) debiasing module with a self-supervised Masked Contrastive Learning (MaskCL) objective that penalizes representational shifts when those tokens are individually masked. 
The procedures operate on unlabeled test data, and the LoRA strength is calibrated on a small support set post-adaptation. 

Across sentiment classification, toxicity detection, and natural language inference, \textsc{Shortcut Guardrail} recovers substantial performance under shortcut distribution shift (i.e., where the shortcut correlations learned during training no longer hold at test time), yielding results comparable to supervised training-time baselines. 
Note that we study token-level shortcuts in pretrained text classifiers, the predominant setting of deployed text classification; extensions to other shortcut types and architectures are discussed in Section~\ref{sec:limitations}. Our code is available at \url{https://anonymous.4open.science/r/shortcut_guardrail_code-D90D}.

Our main contributions are as follows:
\begin{itemize}
    \item We show that biased text encoders internalize a shortcut signal recoverable through unsupervised gradient-based attribution.
    \item We prove that deployment-time shortcut mitigation is information-theoretically upper-bounded by training-time mitigation.
    \item We develop \textsc{Shortcut Guardrail}, a deployment-time shortcut mitigation framework for pretrained text encoders that utilizes gradient signals and contrastive learning, achieving performance comparable with supervised baselines under shortcut distribution shift.
\end{itemize}
\section{Background and Related Work}
Shortcut learning, where models rely on superficial patterns rather than task-relevant features, is well-documented in NLP~\citep{geirhos2020shortcut, du2022shortcut_survey}. Token-level shortcuts, such as annotation artifacts~\citep{gururangan2018annotation, mccoy2019right} and spurious word-label correlations in text classification~\citep{wang2022identifying, du2023less_learn_shortcut, zhou2024navigating}, are particularly prevalent. 

Existing mitigation methods generally operate during training and require auxiliary information: Group DRO ~\citep{sagawa2020distributionally} needs shortcut annotations, JTT~\citep{liu2021just_train_twice} involves full-model retraining by upweighting misclassified samples, DFR~\citep{kirichenko2023last_layer} replaces the classification head with logistic regression, and NFL~\citep{chew2024nfl} applies regularization to prevent deviation from an untrained base model. Recent methods such as InterpoLL~\citep{korakakis2025mitigating} and FairFlow~\citep{cheng2024fairflow} further reduce supervision via representation
interpolation and undecided learning, respectively. 
Closely related to our approach, LTGR~\citep{du2021towards} uses gradient signals to identify influential tokens during training, and C2L~\citep{choi2022c2l} masks attribution-selected tokens to build contrastive counterfactual pairs. 
These methods cannot adapt when only a converged model and unlabeled test data remain.
We position in detail the technical novelty of our approach in Appendix~\ref{app:positioning}.

Test-time adaptation (TTA) targets deployment-time distribution shift but is largely built for vision architectures~\citep{wang2021tent, niu2022eata, sun2020test}, relying on batch normalization and continuous inputs that do not transfer to text encoders such as BERT~\citep{zhao2023pitfalls}. We discuss related work in detail in Appendix~\ref{app:related_work}.

\section{Methods}

\begin{figure*}[ht]
    \centering
    \includegraphics[width=0.94\textwidth]{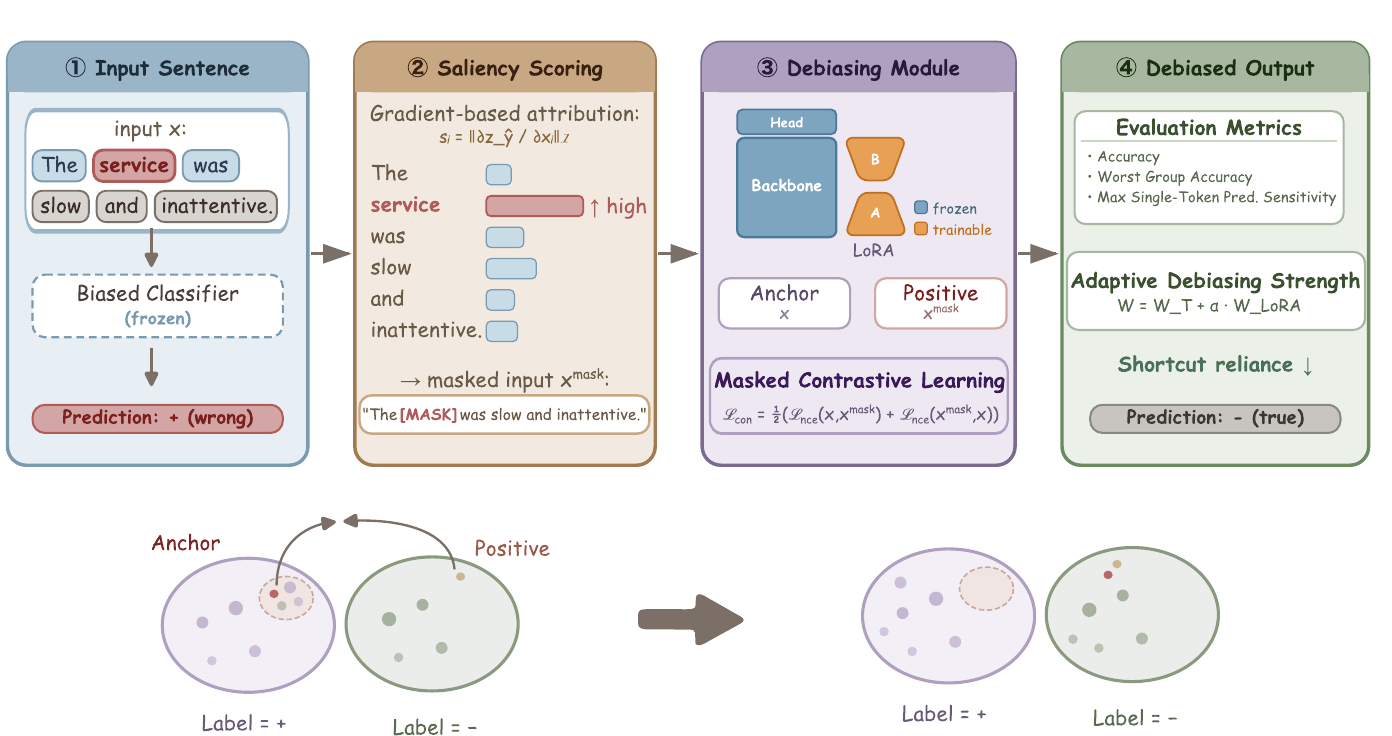}
    \caption{Overview of \textsc{Shortcut Guardrail}, which
(1) obtains predictions from a frozen biased text classifier,
(2) captures shortcut tokens via gradient-based saliency scoring,
(3) trains a lightweight LoRA adapter via Masked Contrastive Learning (MaskCL),
and (4) calibrates the debiasing strength $\alpha$ to produce debiased predictions with reduced shortcut reliance. The bottom panel illustrates the effect of MaskCL training. Before training, spurious correlations cause examples containing the shortcut token (Anchor, $x$) to form a cluster within the wrong label, while their masked variants (Positive, $x^{\text{mask}}$) are projected to the correct label. The MaskCL objective pulls anchors and positives closer in the embedding space, dissolving the spurious cluster and reducing shortcut reliance.}
\label{fig: overview}
\end{figure*}

\subsection{Problem Formulation}
Let $f_\theta$ denote a pretrained text encoder with parameters $\theta$, pretrained on a labeled dataset $\mathcal{D}_{\text{train}}$. 
At deployment time, we receive a batch of unlabeled inputs $\mathcal{D}_{\text{test}} = \{x_1, x_2, \ldots, x_n\}$ from a distribution where predictive token-label correlations during training may not hold. 
We define a token $t$ as a \textit{shortcut token}~\citep{wang2022identifying, chew2024nfl} if $P_\theta(y = c \mid t \in x)$ is disproportionately high for a particular class $c$. Our goal is to adapt the model's predictions to reduce reliance on such shortcut tokens, without access to $\mathcal{D}_{\text{train}}$.

\subsection{Gradient-Based Token Attribution} \label{sec:saliency}

For each input $x = (t_1, t_2, \ldots, t_L)$ of length $L$, we compute a saliency score $s_i$ for each token $t_i$ using gradient-based attribution:
$$
s_i = \| \frac{\partial \mathcal{L}_{\mathrm{cls}}(f_\theta(x), \hat{y})}{\partial \mathbf{e}_i} \odot \mathbf{e}_i \|,
$$
where $\mathbf{e}_i$ is the learned embedding of token $t_i$, $\odot$ denotes element-wise multiplication, $\|\cdot\|$ represents the $\ell_2$ norm, and $\mathcal{L}_{\mathrm{cls}}$ is the cross-entropy loss between the model's output logits $f_\theta(x)$ and its own prediction label $\hat{y}$. This corresponds to the gradient $\times$ input attribution method ~\citep{sundararajan2017axiomatic, li2016understanding}, measuring each token's contribution to the classification output.

We select the tokens with the top-$k$ saliency scores to form the \textit{important token set} $\mathcal{H}(x) = \{t_{i_1}, \ldots, t_{i_k}\}$, which function as shortcut token candidates (not necessarily confirmed ones).
We set $k=10$ by default: for $99\%$ of the shortcut-containing samples, the top-10 important tokens capture the shortcut token (Figure~\ref{fig: important_strength}, Section~\ref{sec:tokens in important tokens}), under a strong spurious correlation (p=$100\%$), while keeping computation moderate.

\subsection{Debiasing Module} \label{sec:debiasing}
\paragraph{Architecture.} We insert LoRA adapters~\citep{hu2022lora} into the backbone encoder of $f_\theta$ while the classification head remains frozen. The LoRA parameters $\phi$ are the only trainable components. Let $z_\phi(x)$ denote the representation of the \texttt{[CLS]}  token produced by the LoRA-adapted encoder.

\paragraph{Anchor-Positive Pairs.} Unlike standard contrastive learning approaches, e.g., SimCLR~\citep{chen2020simclr}, that use random augmentations such as dropout or back-translation, we construct anchor-positive pairs through attribution-guided masking: 

For each $t_{i_j} \in \mathcal{H}(x)$, we produce a masked variant of $x$, denoted $x^{\text{mask}_j}$, by replacing $t_{i_j}$ with the \texttt{[MASK]} token while retaining all other tokens:
\[
x^{\text{mask}_j} = (t_1, \ldots, t_{i_j-1}, 
\texttt{[MASK]}, t_{i_j+1}, \ldots, t_L).
\]

We pair the original input $x$ (``anchor'') with each of its masked variants $x^{\text{mask}_j}$ (``positive''), yielding $k$ anchor-positive pairs per input. This design explicitly probes the model's reliance on each candidate shortcut token independently, rather than enforcing invariance to arbitrary perturbations.

\paragraph{Masked Contrastive Learning (MaskCL).} For a batch of $B$  inputs, we obtain $B \times k$ anchor-positive pairs $\{(x_i,\, x_i^{\text{mask}_j})\}_{i=1,\,j=1}^{B,\,k}$. We first employ the LoRA-adapted model to produce $\ell_2$-normalized embeddings:
\begin{align*}
    \mathbf{a}_i &= z_\phi(x_i)/\|z_\phi(x_i)\|, \\
    \mathbf{p}_{ij} &= z_\phi(x_i^{\text{mask}_j}) / \|z_\phi(x_i^{\text{mask}_j})\|.
\end{align*}
$(\mathbf{a}_i,\mathbf{p}_{ij})$ are the anchor-positive embeddings, while all other embeddings $\mathbf{a}_{i'} \ (i'\neq i, i'\in\{1, \cdots, B\})$ in the batch function as negative samples. The MaskCL contrastive loss $\mathcal{L}_{\text{con}}$ is defined as follows:
\begin{align*}
    & \mathcal{L}_{\text{nce}}^{x, x^{\text{mask}}} = -\sum_{i=1}^{B} \sum_{j=1}^{k} 
    \log \frac{\exp(\mathbf{a}_i^\top \mathbf{p}_{ij} / \tau)}
    {\sum_{i'=1}^{B} \exp(\mathbf{a}_{i'}^\top \mathbf{p}_{ij} / \tau)},\\
    & \mathcal{L}_{\text{nce}}^{x^{\text{mask}}, x} = -\sum_{i=1}^{B} \sum_{j=1}^{k} 
    \log \frac{\exp(\mathbf{a}_i^\top \mathbf{p}_{ij} / \tau)}
    {\sum_{i'=1}^{B} \exp(\mathbf{a}_{i}^\top \mathbf{p}_{i'j} / \tau)},\\
    & \mathcal{L}_{\text{con}} = \frac{1}{2Bk} (\mathcal{L}_{\text{nce}}^{x, x^{\text{mask}}} + \mathcal{L}_{\text{nce}}^{x^{\text{mask}}, x}),
\end{align*}

where $\tau$ is a hyperparameter controlling the concentration of the similarity distribution over negative samples (default to $0.1$ as in the SimCLR implementation~\citep{chen2020simclr}). 
MaskCL encourages the model to produce consistent representations regardless of the presence of individual important tokens. 
It pulls anchor-positive pairs closer in the embedding space, dissolving spurious clusters formed around shortcut tokens and reducing reliance on each potential shortcut (Figure~\ref{fig: overview}).

\subsection{Calibrating Debiasing Strength} \label{sec:alpha_lora}
After training the LoRA adapter, the degree of intervention is controlled by the debiasing strength $\alpha$. The final model weights are computed as:
\[W = W_T + \alpha \cdot W_{\text{LoRA}}\]
where $W_T$ and $W_{\text{LoRA}}$ denote the original frozen weights and the learned LoRA weights, respectively. 
The optimal $\alpha$ depends on the degree of distribution shift at deployment time: when shortcut-label correlations learned from training remain valid, a small $\alpha$ preserves useful signals; when these correlations break down, a larger $\alpha$ is needed to suppress harmful shortcuts.  
Instead of defaulting to a fixed $\alpha$, we calibrate $\alpha$ via a grid search over $\{0, 0.1, \ldots, 1.0\}$ using a labeled support set as a proxy of the target distribution. 

We also show that the performance with adaptively calibrated $\alpha$ is stable for support set size $\ge 40$ and matches or improves on the fixed $\alpha=1$ baseline (Figure~\ref{fig:size_support} in Appendix~\ref{app:support_set}).
We thus use $40$ support examples in our experiments. 
\section{Models Encode a Recoverable Signal of Their Shortcuts} \label{sec:finding}
We establish our key finding: converged models internalize token-level shortcuts (Section~\ref{sec:analysis}), and gradient-based attribution reliably captures the responsible shortcut tokens (Section~\ref{sec:tokens in important tokens}).

\subsection{Models Internalize Token-Level Shortcuts} \label{sec:analysis}


We first show that models are susceptible to shortcuts encountered during training, which undermines deployment-time performance. We construct a controlled testbed using Amazon Product Reviews~\citep{ni2019justifying} for binary sentiment classification, following the protocol of~\citet{chew2024nfl}, where a token (``book'') is injected such that $P(y = 1 | \text{``book''} \in x) = 1$ in the training set. We evaluate on an unfiltered test set across four groups defined by label $y$ and the existence of the shortcut token ``book'' in the input (see Appendix~\ref{testbed} for details).
Group~4, with negative samples ($y=0$) containing the shortcut token, 
directly exposes shortcut reliance: a model that has learned the association (``book'' $\rightarrow y= 1$) will systematically misclassify these samples as $\hat{y}=1$.  

\begin{figure}[ht]
  \includegraphics[width=\columnwidth]{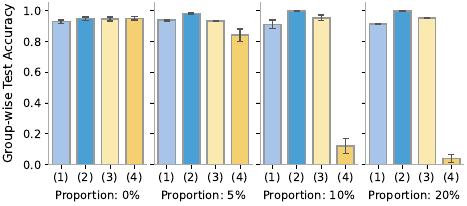}
  \caption{Group-wise test accuracy under different proportions of training data with the shortcut token ``book''. The strength of the spurious correlation $P(y=1 | \text{``book''} \in x)$ is $1$ and label distributions are balanced ($P(y=1)=0.5$). Error bars denote standard deviation over three random trials. (1)-(4) correspond to Group 1-Group 4 of the test data. See Table~\ref{tab:fig:learn_spurious} in Appendix~\ref{testbed} for group-wise training data statistics.}
  \label{fig: learn_spurious}
\end{figure}

\begin{figure}[ht]
  \includegraphics[width=\columnwidth]{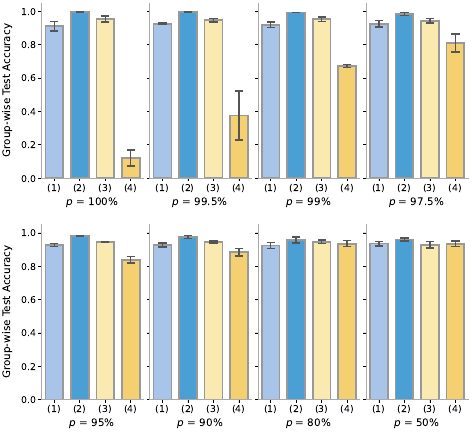}
  \caption{Group-wise test accuracy under different strengths of the spurious correlation, where $p = P(y=1 | \text{``book''} \in x)$. The proportion of training samples containing the shortcut ``book'' is fixed at 10\%, and label distributions are balanced ($P(y=1)=0.5$). (1)-(4) correspond to Group 1-Group 4 of the test data. See Table~\ref{tab:fig:spurious_strength} in Appendix~\ref{testbed} for training data statistics.}
  \label{fig: spurious_strength}
\end{figure}

As shown in Figure~\ref{fig: learn_spurious}, even a small fraction of training data with the shortcut can lead the model to rely on this spurious correlation, resulting in a low test accuracy on Group 4.
We further observe that such reliance increases monotonically with the strength of the spurious correlation, i.e., $P(y=1 | \text{``book''} \in x)$, in the training set (Figure~\ref{fig: spurious_strength}). 
These results suggest that a handful of training data with a strong spurious correlation suffices to induce reliance on shortcuts, raising a natural question: given that shortcuts are already encoded in the model's parameters, can we detect them without access to training data or shortcut annotations?

\begin{figure}[ht]
  \includegraphics[width=\columnwidth]{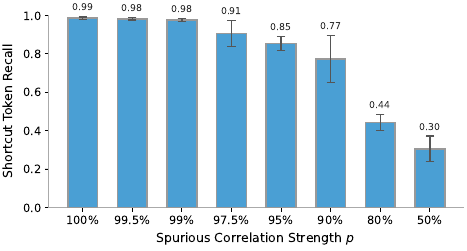}
  \caption{
  Shortcut Token Recall under different strengths of the spurious correlation. Each bar shows the percentage of samples whose shortcut (``book'') appears among the top-10 important tokens, averaged over three random trials.}
  \label{fig: important_strength}
\end{figure}
\begin{figure}[ht]
  \includegraphics[width=\columnwidth]{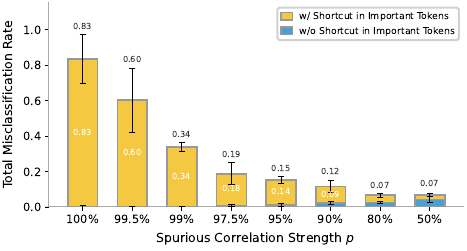}
  \caption{Total misclassification rate under different strengths of the spurious correlation $p$. The misclassification rate (black annotation) can be attributed to two groups of samples: those with the shortcut (``book'') among the important tokens (gold bar, white annotation), and those without (blue bar).}
  \label{fig: important_misclassification}
\end{figure}
\vspace{-6pt}

\subsection{Gradients Localize Shortcut Tokens}\label{sec:tokens in important tokens}
In fact, these shortcut reliances can be pinpointed using the model's internal attribution. We define ``important tokens'' as the tokens with the highest saliency scores for a given sample, and ``Shortcut Token Recall'' as the fraction of samples with a shortcut token appearing among the important tokens. 
As shown in Figure~\ref{fig: important_strength}, Shortcut Token Recall increases with the strength of the spurious correlation, reaching above $90\%$ under strong spurious correlations ($P(y=1 | \text{``book''} \in x)>97.5\%$) using only the top-$10$ important tokens.
Figure~\ref{fig: important_misclassification} further shows that the resulting performance degradation is primarily driven by samples where the shortcut appears among the important tokens.

\section{Deployment-Time Mitigation Performance Ceiling is Upper-Bounded by Training-Time Mitigation}
\label{sec:theory}
We now characterize a fundamental limit: some shortcut information is lost once the model has converged. We prove that deployment-time mitigation is information-theoretically upper-bounded by training-time mitigation. The full formal setup and all proofs are deferred to Appendix~\ref{app:theory}.

Shortcut mitigation involves two stages: identifying which token-label correlations are spurious, and suppressing them. Let $S^*$ denote the true set of such correlations, modeled as a random variable taking values in a finite set $\mathcal{S}$ ($|\mathcal{S}| \geq 2$). 
Training- and deployment-time methods differ in their estimate of $S^*$, denoted $\hat{S}_{\mathrm{train}}$ and $\hat{S}_{\mathrm{deploy}}$: the former may access the training set $\mathcal{D}_{\mathrm{train}}$ directly, while the latter can only use the converged parameters $\theta$. 
We assume that $S^* \rightarrow \mathcal{D}_{\mathrm{train}} \rightarrow \theta \to g(\theta, \mathcal{D}_{\mathrm{test}})$ is a Markov chain, where the converged model $\theta$ acts as an information bottleneck between $\mathcal{D}_{\mathrm{train}}$ and deployment-time observables $g(\theta, \mathcal{D}_{\mathrm{test}})$.

\begin{restatable}[Information Upper Bound]{theorem}{thmInfoBound}
\label{thm:dpi}
The mutual information between the true spurious correlation structure $S^*$ and any deployment-time estimate $\hat{S}_{\mathrm{deploy}}$ satisfies:
\[
I(S^*;\, \hat{S}_{\mathrm{deploy}}) \;\leq\; I(S^*;\, \theta) \;\leq\; I(S^*;\, \mathcal{D}_{\mathrm{train}}).
\]
In particular, the model parameters lose $\Delta I = I(S^*;\, \mathcal{D}_{\mathrm{train}}) - I(S^*;\, \theta) \geq 0$ units of information on $S^*$ relative to the training set $\mathcal{D}_{\mathrm{train}}$.
\end{restatable}

Theorem~\ref{thm:dpi} bounds information, not performance.
We bridge this via Fano's Inequality. 
Denoting the Shannon entropy of $S^*$ by $H(S^*)$, we have:


\begin{restatable}[Identification Error Bounds]{theorem}{fanoErrBound}
\label{thm:fano}
Let $S^*$ take values in the finite set $\mathcal{S}$ with $|\mathcal{S}| \geq 2$. Define the Fano lower bounds on identification error:
\begin{align}
    L_{\mathrm{deploy}} &= \frac{H(S^*) - I(S^*;\, \theta) - 1}{\log(|\mathcal{S}| - 1)}, \label{eq:L_deploy} \\[4pt]
    L_{\mathrm{train}}  &= \frac{H(S^*) - I(S^*;\, \mathcal{D}_{\mathrm{train}}) - 1}{\log(|\mathcal{S}| - 1)}. \label{eq:L_train}
\end{align}
Then the misidentification probabilities satisfy:
\begin{align}
    P(\hat{S}_{\mathrm{deploy}} \neq S^*) &\;\geq\; L_{\mathrm{deploy}}, \label{eq:fano_deploy} \\
    P(\hat{S}_{\mathrm{train}} \neq S^*)  &\;\geq\; L_{\mathrm{train}}, \label{eq:fano_train}
\end{align}
and $L_{\mathrm{deploy}} \geq L_{\mathrm{train}}$.
\end{restatable}

Finally, we bound the training- and deployment-time mitigation performance ceilings (not the performance of any particular method):

\begin{restatable}[Monotonicity]{assumption}{assumpmono}
\label{assump:mono}
There exists a non-decreasing, differentiable function $\rho:[0,1] \to \mathbb{R}$ such that a method with identification error $L$ has expected mitigation performance $\rho(1 - L)$.
\end{restatable}

\begin{restatable}[Performance Ceiling from Identification Error]{lemma}{lemmaCeilingError}
    \label{lem:perf_bound}
Under Assumption~\ref{assump:mono}, any mitigation method whose identification error satisfies $P(\hat{S} \neq S^*) \geq L$ for some $L \in [0,1]$ has expected performance bounded by:
\[
\mathbb{E}[R(\hat{S}, S^*)] \;\leq\; \rho(1 - L).
\]
In particular, combining with Theorem~\ref{thm:fano}:
\begin{align}
\mathbb{E}[R(\hat{S}_{\mathrm{deploy}}, S^*)] &\;\leq\; \rho(1 - L_{\mathrm{deploy}}), \label{eq:perf_deploy} \\
\mathbb{E}[R(\hat{S}_{\mathrm{train}}, S^*)]  &\;\leq\; \rho(1 - L_{\mathrm{train}}). \label{eq:perf_train}
\end{align}
\end{restatable}

\begin{restatable}[Performance Ceiling Comparison]{theorem}{thmPerfBound}
\label{thm:main}
Under Assumption~\ref{assump:mono}, the performance ceilings for deployment- and training-time mitigation satisfy:
\[
\rho(1 - L_{\mathrm{deploy}}) \;\leq\; \rho(1 - L_{\mathrm{train}}),
\]
Moreover, the gap between the two bounds is:
\[
\rho(1 - L_{\mathrm{train}}) - \rho(1 - L_{\mathrm{deploy}}) \;=\; \rho'(\xi) \cdot \frac{\Delta I}{\log(|\mathcal{S}| - 1)},
\]
for some $\xi \in [1 - L_{\mathrm{deploy}},\; 1 - L_{\mathrm{train}}]$, where $\Delta I = I(S^*;\, \mathcal{D}_{\mathrm{train}}) - I(S^*;\, \theta) \geq 0$.
\end{restatable}

Theorem~\ref{thm:main} shows that the more information $f_\theta$ loses about $S^*$, the larger the inherent gap between the training- and deployment-time mitigation performance ceilings. 
We next show how gradient attribution (Section~\ref{sec:finding}) guides deployment-time shortcut mitigation empirically in Section~\ref{sec:main}.

\section{\textsc{Shortcut Guardrail} Mitigates Token-Level Shortcuts} \label{sec:main}

\subsection{Experiment Setup}
We describe the datasets and baselines used in our experiments. We summarize the sizes of the datasets in Table~\ref{tab:dataset_sizes}.  Implementation details are provided in Appendix~\ref{app:implementation}.

\begin{table}[htp]
\centering
\small
\setlength{\tabcolsep}{6pt}
\begin{tabular}{@{}llccc@{}}
\toprule
Metric & Model & SST-2 & Civil & MultiNLI \\
\midrule
\multirow{5}{*}{Accuracy}
& ERM  & 0.911  & \textbf{0.880}  & 0.819 \\
& JTT  & \textit{0.915} &  \textit{0.878} &  \textbf{0.829} \\
& NFL  & 0.815 &  0.832 &  0.665 \\
& DFR  & \textit{0.915} &  \textit{0.878} &   \textit{0.824} \\
& SG (Ours) & \textbf{0.917} & 0.874  & 0.816 \\
\midrule
\multirow{5}{*}{WGA}
& ERM  & {--} & 0.737 & 0.429  \\
& JTT  & {--} &  \textit{0.788} & 0.429 \\
& NFL  & {--} &  0.717 & 0.428  \\
& DFR  & {--} &  0.778  & \textit{0.571} \\
& SG (Ours)  & {--} & \textbf{0.814} & \textbf{0.714} \\
\bottomrule
\end{tabular}
\caption{Performance on real-world benchmarks with naturally occurring token-label correlations. The best results are marked in \textbf{bold} while the second best are in \textit{italics}. SST-2 does not have predefined shortcut annotations, hence WGA is not applicable ({--}).}
\label{tab:realworld_results}
\end{table}
\begin{table*}[t]
\centering
\small
\setlength{\tabcolsep}{6pt}
\begin{tabular}{@{}llcccccc@{}}
\toprule
 Metric & Model & Civil & MultiNLI & Yelp-ST & Yelp-Syn & GoEmo-ST & GoEmo-Syn \\
\midrule
\multirow{5}{*}{Accuracy}
& ERM  & 0.522 & 0.270 & 0.355 & 0.419 & 0.513 & 0.593 \\
& JTT   & 0.532 & 0.268 & 0.440 & \textit{0.490} & 0.403 & \textit{0.613} \\
& NFL   & \textit{0.536} & 0.048 & \textit{0.449} & 0.471 & 0.387 & 0.377 \\
& DFR   & \textit{0.536} & \textit{0.273} & 0.373 & 0.435 & \textit{0.557} & \textbf{0.630} \\
& SG (Ours) & \textbf{0.572} & \textbf{0.323} & \textbf{0.488} & \textbf{0.530} & \textbf{0.627} & 0.607 \\
\midrule
\multirow{5}{*}{WGA}
& ERM   & 0.216 & 0.010 & 0.030 & 0.035 & 0.026 & \textit{0.222} \\
& JTT   & 0.208 & \textbf{0.030} & \textit{0.266} & \textbf{0.333} & 0.000 & 0.000 \\
& NFL   & \textit{0.260} & \textbf{0.030} & 0.157 & 0.180 & 0.000 & 0.000 \\
& DFR   & 0.240 & 0.010 & 0.040 & 0.206 & \textit{0.139} & \textbf{0.359} \\
& SG (Ours)   & \textbf{0.396} & 0.010 & \textbf{0.353} & \textit{0.220} & \textbf{0.339} & \textit{0.222} \\
\bottomrule
\end{tabular}
\caption{Overall accuracy and worst-group accuracy (WGA) under shortcut distribution shift. The best results are marked in \textbf{bold} while the second best are marked in \textit{italics}. The CivilComments and MultiNLI datasets in this table are reconstructed with controlled shortcut injection and differ from the original versions in Table~\ref{tab:realworld_results}.}
\label{tab:ood_results}
\end{table*}
\vspace{-6pt}


\begin{table*}[t]
\centering
\small
\setlength{\tabcolsep}{6pt}
\begin{tabular}{@{}lccccc@{}}
\toprule
Dataset & ERM & JTT & NFL & DFR & SG (Ours) \\
\midrule
Civil     & 0.530 & 0.549 {\scriptsize($+$0.019)} & 0.451 {\scriptsize($-$0.079)} & 0.526 {\scriptsize($-$0.004)} & 0.438 {\scriptsize($-$0.092)} \\
MultiNLI  & 0.797 & 0.267 {\scriptsize($-$0.530)} & 0.143 {\scriptsize($-$0.654)} & 0.212 {\scriptsize($-$0.585)} & 0.381 {\scriptsize($-$0.416)} \\
Yelp-ST   & 0.331 & 0.393 {\scriptsize($+$0.062)} & 0.074 {\scriptsize($-$0.257)} & 0.339 {\scriptsize($+$0.008)} & 0.203 {\scriptsize($-$0.128)} \\
Yelp-Syn  & 0.219 & 0.264 {\scriptsize($+$0.045)} & 0.083 {\scriptsize($-$0.136)} & 0.139 {\scriptsize($-$0.080)} & 0.126 {\scriptsize($-$0.093)} \\
GoEmo-ST  & 0.647 & 0.355 {\scriptsize($-$0.292)} & 0.102 {\scriptsize($-$0.563)} & 0.713 {\scriptsize($+$0.066)} & 0.501 {\scriptsize($-$0.146)} \\
GoEmo-Syn & 0.488 & 0.180 {\scriptsize($-$0.308)} & 0.190 {\scriptsize($-$0.263)} & 0.527 {\scriptsize($+$0.039)} & 0.414 {\scriptsize($-$0.074)} \\
\bottomrule
\end{tabular}
\caption{Shortcut reliance as measured by MSTPS. Parentheses show the change relative to ERM. 
}
\label{tab:mstps}
\end{table*}
\vspace{-6pt}

\paragraph{Datasets.} We evaluate our framework under two complementary settings. The first uses real-world benchmarks where shortcut correlations arise naturally from the data collection process. To ensure comparable scale across all benchmarks while preserving original label distributions, we use stratified subsets of (1) SST-2~\citep{wang2019glue} for binary sentiment classification; (2) CivilComments~\citep{borkan2019nuanced} for toxicity detection, where demographic identity terms are spuriously correlated with the toxic label; and (3) MultiNLI~\citep{williams2018multinli} for natural language inference, where negation words frequently serve as a shortcut for the contradiction class~\citep{gururangan2018annotation, mccoy2019right}.

The second employs controlled benchmarks constructed following the protocol of~\citet{zhou2024navigating}, which injects token-level shortcuts with controllable strengths into both training and test data. We evaluate six configurations: Yelp-ST and Yelp-Syn for sentiment classification~\citep{zhang2015character}, GoEmo-ST and GoEmo-Syn for emotion classification on GoEmotions~\citep{demszky2020goemotions}, and modified versions of CivilComments and MultiNLI that we constructed. Here, ST (single-token) configurations introduce a label-irrelevant token (e.g., ``honestly'') with a label-dependent probability. The Syn (synonym) configurations inject shortcut tokens from a set of 15 synonyms (e.g., ``candidly'', ``in truth'').

The shortcut strength is governed by a parameter $\lambda$ that represents the label-conditional probability of the shortcut's occurrence. Note that $\lambda$ is distinct from the spurious correlation strength $p$ used in Sections~\ref{sec:analysis} and~\ref{sec:tokens in important tokens}. We train all models with $\lambda = 1$ to simulate the strongest possible spurious correlation. We use the ``anti-test set'' protocol of \citet{zhou2024navigating}, which reverses the shortcut-label distribution of the training set to maximally expose shortcut reliance. Full construction details are provided in Appendix~\ref{app:shortcut_construction}.

\paragraph{Baselines.}
We compare our framework against Empirical Risk Minimization (ERM, vanilla fine-tuning), and training-time shortcut mitigation methods: JTT \citep{liu2021just_train_twice}, NFL \citep{chew2024nfl}, and DFR \citep{kirichenko2023last_layer}. 
Implementation details are provided in Appendix~\ref{app:baselines}.


\subsection{Evaluation Metrics}

\paragraph{Accuracy and Worst-Group Accuracy (WGA).} We report overall test accuracy and worst-group accuracy (WGA). We define groups for all datasets with shortcut annotations (Yelp, GoEmo, CivilComments, and MultiNLI) based on all possible combinations of label and shortcut existence (Appendix~\ref{testbed}).
As the minimum accuracy across all groups, WGA reflects the model's robustness to shortcut distribution shifts: a model may achieve high overall accuracy by exploiting shortcuts, but will yield low accuracy on ``conflicting'' groups where those shortcut-label correlations observed during training do not hold.

\paragraph{Maximum Single-Token Prediction Sensitivity (MSTPS).} 
To directly quantify the model's reliance on token-level shortcuts, we propose the metric MSTPS. 
Existing shortcut-degree measures rely on training dynamics or early-training checkpoints~\citep{du2021towards}, which are unavailable in our deployment-time setting; MSTPS instead quantifies shortcut reliance using only the converged model and test data.
For each input $x_i$, we measure the maximum shift in prediction probability caused by masking any individual important token:
\[\text{MSTPS} = \frac{1}{N} \sum_{i=1}^{N} \max_{j} \left| P(\hat{y}_i | x_i) - P(\hat{y}_i | x_i^{\text{mask}_j}) \right|,\]
where $N$ is the number of test samples, $x_i^{\text{mask}_j}$ denotes input 
$x_i$ with the $j$-th important token replaced by \texttt{[MASK]}, $\hat{y}_i$ is the model's predicted label for $x_i$, and the maximum is taken over $\{j|t_j \in \mathcal{H}(x_i)\}$, the top-$k$ tokens identified by saliency scoring (Section~\ref{sec:saliency}). 
A high MSTPS indicates that the model's prediction can be substantially altered by removing a single token, suggesting over-reliance on local shortcuts. Conversely, an effective debiasing method should reduce MSTPS, indicating that the model has learned to integrate broader contextual evidence rather than isolated features.

\subsection{\textsc{Shortcut Guardrail} is Effective Across Shortcut Distributions, Types, and Strengths}
\paragraph{Real-world Benchmarks.}
Table~\ref{tab:realworld_results} reports results on SST-2, CivilComments, and MultiNLI without artificial shortcuts. 
Our method achieves the highest accuracy on SST-2, and substantially raises WGA on CivilComments and MultiNLI at a marginal cost in accuracy. In these benchmarks with no artificial distribution shift, WGA is governed by minority ``conflicting'' groups where shortcut--label correlations do not hold at test time, while overall accuracy is dominated by the majority where they still apply; mitigation redistributes performance toward the minority groups, raising WGA at little cost to accuracy~\citep{liu2021just_train_twice,sagawa2020distributionally}. 
These results confirm that \textsc{Shortcut Guardrail} effectively mitigates reliance on naturally occurring shortcuts.


\paragraph{Controlled Benchmarks with Shortcut Distribution Shift.}
We further evaluate \textsc{Shortcut Guardrail} on controlled benchmarks with artificially injected shifts in the shortcut distribution, maximally exposing LM's reliance on shortcuts. 
As shown in Table~\ref{tab:ood_results}, ERM suffers severe WGA degradation, confirming heavy shortcut reliance. 
In contrast, our method yields notable WGA improvements over ERM and performs comparably or better than training-time baselines, attaining the best accuracy on all datasets except GoEmo-ST and top-2 WGA except on MultiNLI, where all methods struggle. Since we improved WGA on the original MultiNLI (Table~\ref{tab:realworld_results}), this globally poor performance likely stems from the shortcut injection.

\paragraph{Reducing Shortcut Reliance.}
Across all benchmarks, our method consistently reduces MSTPS relative to ERM, confirming that the MaskCL objective redistributes predictive evidence across a broader context (Table~\ref{tab:mstps}).
Lower MSTPS does not by itself guarantee better performance, since some token sensitivity reflects legitimate signal rather than spurious correlation: on Yelp-ST, NFL attains the lowest MSTPS (0.074) but lower accuracy (0.449) and WGA (0.157) than our method (MSTPS 0.203, accuracy 0.488, WGA 0.353). Our method thus balances reducing shortcut reliance against preserving task-relevant information (Figure~\ref{fig: mstps_vs_acc} in Appendix~\ref{app: balance}).

\paragraph{Performance Across Varying Shortcut Strengths and In-distribution Test Data.}

To evaluate our method's robustness across varying degrees of distribution shift, we modulate the injected shortcut strength, denoted by $\lambda$. Table~\ref{tab:lambda} shows results on Yelp-ST and  (modified). 
We observe that the performance gap between our method and ERM widens as the $\lambda$ becomes larger, demonstrating that our approach is more effective when shortcuts are more severe. 
Nevertheless, at lower $\lambda$ values, which more closely resemble naturally occurring distributions with weaker shortcut signals, our method still improves overall accuracy.

We also verify that our method preserves in-distribution performance, where the shortcut-label correlations from training data remain intact (Table~\ref{tab:id_results} in Appendix~\ref{app:id_results}). On most benchmarks, the adaptive $\alpha$ calibration mechanism selects $\alpha = 0$, fully preserving the original ERM predictions. 


\begin{table}[ht]
\centering
\small
\setlength{\tabcolsep}{3.5pt}
\begin{tabular}{@{}llcccccc@{}}
\toprule
Dataset & $\lambda$ & \multicolumn{3}{c}{Accuracy} & \multicolumn{3}{c}{WGA} \\
\cmidrule(lr){3-5} \cmidrule(lr){6-8}
& & ERM & SG & $\Delta$ & ERM & SG & $\Delta$ \\
\midrule
\multirow{3}{*}{Yelp-ST}
& 1.0 & 0.355 & 0.488 & +0.133 & 0.030 & 0.353 & +0.323 \\
& 0.8 & 0.558 & 0.583 & +0.025 & 0.481 & 0.467 & $-$0.014 \\
& 0.6 & 0.614 & 0.616 & +0.002 & 0.506 & 0.480 & $-$0.026 \\
\midrule
\multirow{4}{*}{Civil}
& 1.0 & 0.522 & 0.572 & +0.050 & 0.216 & 0.396 & +0.180 \\
& 0.9 & 0.705 & 0.717 & +0.012 & 0.487 & 0.544 & +0.057 \\
& 0.8 & 0.761 & 0.777 & +0.016 & 0.568 & 0.715 & +0.147 \\
& 0.6 & 0.861 & 0.861 & +0.000 & 0.747 & 0.780 & +0.033 \\
\bottomrule
\end{tabular}
\caption{Effect of shortcut strength $\lambda$ in controlled benchmarks. $\Delta$ = SG $-$ ERM. 
}
\label{tab:lambda}
\end{table}
\vspace{-6pt}
\section{Discussion and Conclusion}
A converged text encoder carries a signal of its learned shortcuts that unsupervised gradient-based attribution can uncover. 
Leveraging this signal, our deployment-time mitigation matches or outperforms training-time baselines.
This is remarkable given that we prove deployment-time mitigation is information-theoretically upper-bounded by training-time methods. 
We attribute this efficacy to the locus of intervention: rather than the dataset-level statistics used by JTT, DFR, and NFL, we use per-input gradient attribution to localize and neutralize shortcut tokens. 
By removing the need for training data, shortcut annotations, and retraining, our framework offers a practical solution to deployment-time distribution shift.
\section*{Limitations} \label{sec:limitations}

While \textsc{Shortcut Guardrail} demonstrates strong performance across diverse benchmarks, we acknowledge the following limitations. (1) As we formally show in Appendix~\ref{app:theory}, deployment-time methods are information-theoretically upper-bounded by training-time methods. Since shortcuts originate from $\mathcal{D}_{\mathrm{train}}$, any deployment-time method can only recover the spurious correlation structure indirectly through the trained model $\theta$, which acts as an information bottleneck. While our method approaches this bound under strong shortcuts (Table~\ref{tab:ood_results}), the gap widens for structurally complex shortcut patterns, writing styles, and syntactic templates. 
(2) Our gradient-based attribution operates at the token level, which limits its ability to capture multi-token shortcut patterns. 
(3) Our evaluation focuses on text classification tasks using BERT-based encoders, which serve as a clean prototype for establishing our findings. Our masked contrastive learning objective is compatible with bidirectional attention in transformer encoders that are pretrained by masked language modeling. It does not naturally extend to decoder-only architectures representative of modern large language models: under causal attention, masking a token leaves the representations of all preceding tokens unchanged, failing to produce the globally distinct representation that our contrastive objective relies on. Extending this framework to decoder-only LMs, which would require rethinking both the saliency attribution step and the contrastive objective, is left for future work.
(4) Although our method eliminates the need for training data access, shortcut annotations, or full model retraining, it still requires a small labeled support set for calibrating the debiasing strength $\alpha$. In fully unsupervised scenarios, alternative calibration strategies would be needed. 
(5) Targeting deployment settings where the model is available but training data and annotations are not, our method assumes access to model weights and gradients. Relaxing this assumption toward black-box/closed-source models is a direction for future work.



\bibliography{custom}

\appendix
\clearpage
\section{Experiment Details} \label{app:exp}
\label{sec:appendix}
\subsection{Dataset Details} \label{app:dataset}
\subsubsection{Testbed dataset}\label{testbed}
We construct a controlled testbed using Amazon Product Reviews~\citep{ni2019justifying} for binary sentiment classification, following the protocol of~\citet{chew2024nfl}. 
Review ratings are on a scale of 1 to 5, where the 4 and 5-star ratings are mapped to a positive sentiment class ($y = 1$) and 1 and 2-star ratings are mapped to a negative sentiment class ($y = 0$), while 3-star ratings are excluded. To inject a shortcut, we filter the training set so that all samples containing a designated token (``book'') are of positive sentiment, i.e., $P(y = 1 \mid \text{``book''} \in x) = 1$. We train a BERT model for sentiment classification on the filtered training set with 10,000 samples, and evaluate on the unfiltered test set, which consists of four distinct groups of equal size:

\begin{itemize}
    \item \textbf{Group 1}: positive samples without the shortcut $(y=1, \text{``book''}\notin x)$.
    \item \textbf{Group 2}: positive samples with the shortcut $(y=1, \text{``book''}\in x)$.
    \item \textbf{Group 3}: negative samples without the shortcut $(y=0, \text{``book''}\notin x)$.
    \item \textbf{Group 4}: negative samples with the shortcut $(y=0, \text{``book''}\in x)$.
\end{itemize}
 
 Note that for all other datasets with shortcut annotations (Yelp, GoEmotions, CivilComments, and MultiNLI), groups are defined analogously across all combinations of label and shortcut presence. Formally, with a label set $Y=\{y_1, y_2, \cdots , y_k\}$ and a shortcut token set $T$, there will be $2k$ groups $C_1, C_2, \cdots, C_{2k}$, where
 \begin{align*}
     &C_{2i-1}=\{(x, y)|y=y_i, \forall t\in T, t\notin x\},\\
     &C_{2i}=\{(x, y)|y=y_i, \exists t\in T, t\in x\}.
 \end{align*}

\paragraph{Details on the effect of shortcut samples proportion.}
For the experiments in Figure~\ref{fig: learn_spurious}, we fix $P(y=1 | \text{``book''} \in x) = 1$ in the training set, and keep an equal proportion of positive and negative samples. Table~\ref{tab:fig:learn_spurious} shows the proportions of training samples in each of the groups above. The test set always contains an equal proportion of all four groups.
\begin{table}[h]
\centering
\caption{The proportion ($/100$) of samples in each group in the training sets for the experiments in Figure~\ref{fig: learn_spurious}.}
\label{tab:fig:learn_spurious}
\begin{tabular}{@{}lcccc@{}}
\toprule
                 & 0\%        & 5\%        & 10\%        & 20\%        \\ \midrule
Group 1          & 50         & 45         & 40          & 30          \\ 
Group 2 & 0 & 5 & 10 & 20 \\
Group 3          & 50         & 50         & 50          & 50          \\
Group 4          & 0          & 0          & 0           & 0           \\ \bottomrule
\end{tabular}
\end{table}

\paragraph{Details on the effect of shortcut strength.}
For the experiments in Figure~\ref{fig: spurious_strength}, we vary the shortcut signal strength $P(y=1 | \text{``book''} \in x)$ while keeping an equal proportion of positive and negative samples, and a proportion of $10\%$ total samples with the shortcut token ``book''. Table~\ref{tab:fig:learn_spurious} shows the proportions of training samples in each of the groups above. The test set always contains an equal proportion of all four groups.
\begin{table}[]
\centering
\caption{The proportion ($/100$) of samples in each group in the training sets for the experiments in Figure~\ref{fig: spurious_strength}. $p = P(y=1 | \text{``book''} \in x)$. The statistics for $p=0.995, 0.975, 0.7$ can be derived analogously.}
\label{tab:fig:spurious_strength}
\begin{tabular}{@{}lccccc@{}}
\toprule
        & $p$=1 & $p$=0.99 & $p$=0.95 & $p$=0.9 & $p$=0.5 \\ \midrule
Group 1 & 40  & 40.1   & 40.5   & 41    & 45    \\
Group 2 & 10  & 9.9    & 9.5    & 9     & 5     \\
Group 3 & 50  & 49.9   & 49.5   & 49    & 45    \\
Group 4 & 0   & 0.1    & 0.5    & 1     & 5     \\ \bottomrule
\end{tabular}
\end{table}

\subsubsection{Real-World Benchmarks}
This section summarizes the three real-world benchmarks used in our evaluation.

\textbf{SST-2} \cite{wang2019glue} is a binary sentiment classification dataset derived from movie reviews. Models trained on SST-2 are known to over-rely on lexical cues such as proper nouns (e.g., \textit{Spielberg}) and sentiment-laden words rather than compositional context \cite{wang2022identifying}. SST-2 does not provide predefined shortcut annotations, so we report only overall accuracy.

\textbf{CivilComments} \cite{borkan2019nuanced} is a binary toxicity detection dataset collected from online comments. Each sample is annotated with demographic identity attributes. These identity terms are spuriously correlated with the toxic label: comments mentioning certain demographic groups are disproportionately labeled toxic even when not inherently harmful. We report both overall accuracy and worst-group accuracy across identity-based groups.

\textbf{MultiNLI} \cite{williams2018multinli} is a three-class natural language inference dataset (classes: \textit{entailment}, \textit{contradiction}, and \textit{neutral}). Known shortcuts include negation words (e.g., \textit{nobody}, \textit{never}) for the contradiction class and high lexical overlap between premise and hypothesis for the entailment class \cite{gururangan2018annotation,mccoy2019right}. We use negation words as the shortcut tokens and report both overall accuracy and worst-group accuracy.

\subsubsection{Controllable benchmark Construction Details }\label{app:shortcut_construction}

We follow the shortcut benchmark of \citet{zhou2024navigating} to construct occurrence-based shortcuts for the Yelp and GoEmotions datasets. Two types of shortcuts are used:

\paragraph{Single Token (ST).}
A single label-irrelevant word, ``honestly'', is inserted at the beginning of a randomly chosen sentence in each review. The insertion probability is label-dependent, controlled by the shortcut strength parameter $\lambda$. For example, in Yelp-ST with $\lambda=1$, the shortcut occurrence/insertion probabilities are 0\%, 25\%, 50\%, 75\%, and 100\% for ratings 1 through 5, creating a spurious correlation between the presence of ``honestly'' and higher ratings. We detail the mapping between $\lambda$ and this probability distribution below (see ``Shortcut Strength to Shortcut Occurrence Probability''). 

\paragraph{Synonym (Syn).}
The construction follows the same procedure as ST, except that instead of always inserting ``honestly'', a word or phrase is uniformly sampled from a set of 15 synonyms: \{``honestly'', ``to be honest'', ``frankly speaking'', ``to tell the truth'', ``to be frank'', ``in truth'', ``candidly'', ``speaking candidly'', ``plainly speaking'', ``to be direct'', ``to come clean'', ``to put it frankly'', ``if I'm being honest'', ``in plain terms'', ``directly speaking''\}. This tests whether models can recognize and exploit the correlation between a set of semantically related terms and the label.

\paragraph{CivilComments and MultiNLI.}
Both datasets come with existing shortcut annotations. CivilComments provides demographic identity term annotations (e.g., \textit{male, female, White, Black, Muslim}). These identity terms collectively form the shortcut token set. MultiNLI provides negation words (\textit{nobody, no, never, nothing}) that are spuriously correlated with the contradiction class. 
We define the groups for WGA calculation using those shortcut token sets following the definition in Appendix~\ref{testbed}.
Training and shifted test sets are constructed using a similar $\lambda$-based scheme of \citet{zhou2024navigating}, but instead of inserting shortcut tokens, we randomly sample instances from each group of the original training/test set, where the proportion of samples from each group is determined by the shortcut occurrence probability distribution given by $\lambda$: For example, for CivilComments, $\lambda=1$ yields a 0\% samples with shortcuts for non-toxic samples, and 100\% samples with shortcuts for toxic samples.

\paragraph{Shortcut Strength to Shortcut Occurrence Probability.}
Across all datasets, the shortcut strength is controlled by a parameter $\lambda \in [0, 1]$. For a $C$-class task with labels $\{1, 2, \ldots, C\}$, the probability of shortcut occurrence for label $c$ is $\lambda(c-1)/(C-1)$. For example, on Yelp ($C=5$, $\lambda=1$), the shortcut insertion probabilities are 0\%, 25\%, 50\%, 75\%, and 100\% for ratings 1 through 5, respectively. Larger $\lambda$ corresponds to stronger spurious correlations. We train with $\lambda = 1$ unless otherwise stated. Table~\ref{tab:shortcut_prob} lists the base probabilities for each dataset.
\begin{table}[h]
\centering
\small
\begin{tabular}{@{}lcl@{}}
\toprule
Dataset & $C$ & Per-class Shortcut Occurrence Probabilities \\
\midrule
Yelp & 5 & 0\%, 25\%, 50\%, 75\%, 100\% \\
GoEmotions & 4 & 0\%, 33.3\%, 66.7\%, 100\% \\
CivilComments & 2 & 0\%, 100\% \\
MultiNLI & 3 & 100\%, 50\%, 0\% \\
\bottomrule
\end{tabular}
\caption{Base shortcut insertion probabilities per label ($\lambda = 1$). Shifted test sets reverse the order of these probabilities (e.g., 0\%, 100\% $\to$ 100\%, 0\%).}
\label{tab:shortcut_prob}
\end{table}

\paragraph{Shifted test sets.}
For all datasets, we generate shifted test sets by reversing the base probability distribution. On Yelp, the shifted test probabilities become 100\%, 75\%, 50\%, 25\%, and 0\% for ratings 1 through 5. This maximally exposes models that rely on the shortcut, as the spurious correlation from training is inverted.

\begin{table*}[t]
\centering
\small
\caption{Dataset sizes. Statistics for ``Controlled'' apply to both shortcut distribution shift and in-distribution datasets.}
\label{tab:dataset_sizes}
\resizebox{\textwidth}{!}{%
\begin{tabular}{@{}lccccccccc@{}}
\toprule
 & \multicolumn{3}{c}{\textbf{Real-World}} & \multicolumn{6}{c}{\textbf{Controlled}} \\
\midrule
 & \textbf{SST-2} & \textbf{Civil} & \textbf{MultiNLI} & \textbf{Civil} & \textbf{MultiNLI} & \textbf{Yelp-ST} & \textbf{Yelp-Syn} & \textbf{GoEmo-ST} & \textbf{GoEmo-Syn} \\
\midrule
Train & 67,349 & 115,931 & 206,175 & 111,298 & 44,508 & 8,900 & 9,400 & 1,124 & 1,124 \\
Test  & 872    & 1,124   & 1,000   & 1,000   & 600    & 1,000 & 1,000 & 685   & 685   \\
\bottomrule
\end{tabular}
}
\end{table*}

\subsection{Balance Between Shortcut Reliance and Task-relevant information}
\label{app: balance}
In Figure~\ref{fig: mstps_vs_acc}, we plot MSTPS against accuracy across all methods and benchmarks reported in Table~\ref{tab:mstps}. We can see that \textsc{Shortcut Guardrail} generally occupies the upper-left portion of the plot, striking a balance between reducing reliance on shortcuts (lower MSTPS than ERM) and preserving task-relevant information (higher accuracy than ERM). 
\begin{figure}[ht]
  \includegraphics[width=\columnwidth]{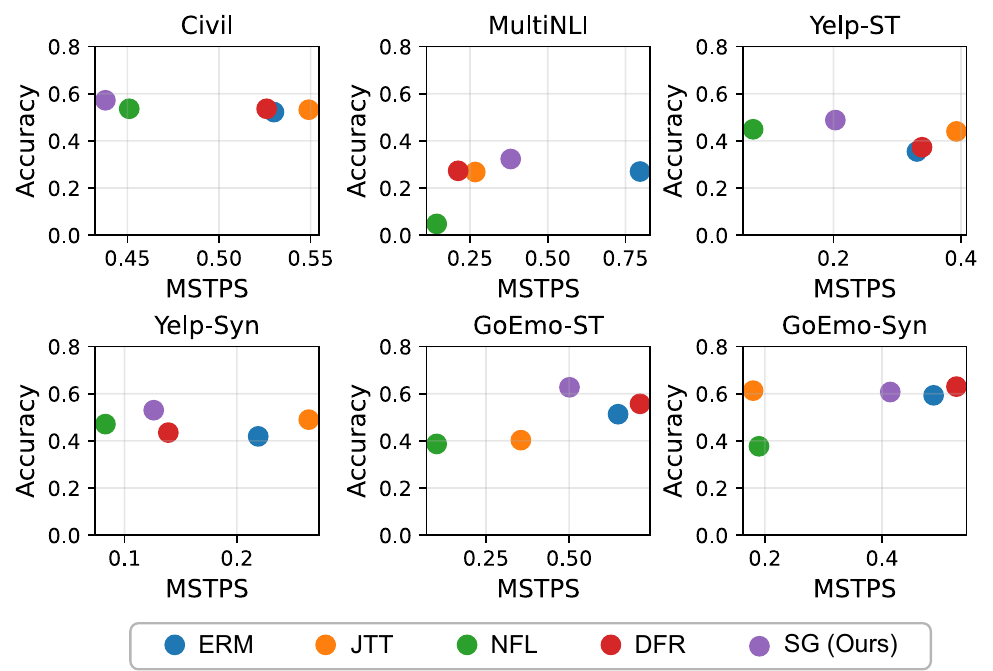}
  \caption{ Scatterplots comparing accuracy with MSTPS across all benchmarks and models reported in Table~\ref{tab:mstps}.
  }
  \label{fig: mstps_vs_acc}
\end{figure}

\subsection{In-Distribution Performance} 
\label{app:id_results}

\begin{table*}[t]
\centering
\small
\setlength{\tabcolsep}{5pt}
\begin{tabular}{@{}llcccccc@{}}
\toprule
Metric & Model & Civil & MultiNLI & Yelp-ST1 & Yelp-SYN1 & GoEmo-ST1 & GoEmo-SYN1 \\
\midrule
\multirow{2}{*}{Accuracy}
& ERM  & 0.980 & 0.908 & 0.671 & 0.679 & 0.896 & 0.928 \\
& SG   & 0.980 & 0.908 & 0.660 & 0.679 & 0.896 & 0.928 \\
\midrule
\multirow{2}{*}{WGA}
& ERM  & 0.964 & 0.840 & 0.376 & 0.000 & 0.667 & 0.774 \\
& SG   & 0.964 & 0.840 & 0.474 & 0.000 & 0.667 & 0.774 \\
\bottomrule
\end{tabular}
\caption{In-distribution (ID) performance. Our method preserves the accuracy of the original ERM model.}
\label{tab:id_results}
\end{table*}

Table~\ref{tab:id_results} reports in-distribution performance, where shortcut-label correlations from training still hold. On most benchmarks, the adaptive $\alpha$ calibration selects $\alpha = 0$, preserving the original ERM predictions. The only exception is Yelp-ST1, where $\alpha = 0.6$ is selected, resulting in a slight accuracy decrease (0.671 $\to$ 0.660) but improved WGA (0.376 $\to$ 0.474).

Yelp-SYN1 shows WGA of 0.000 for both ERM and SG. Under $\lambda=1$, the minority group contains very few examples, and the model consistently misclassifies them due to strong shortcut reliance. This is expected behavior under in-distribution evaluation and does not reflect a failure of our method.

\subsection{Impact of the Size of the Support Set}
\label{app:support_set}
As discussed in Section~\ref{sec:alpha_lora}, the optimal $\alpha$ depends on the degree of distribution shift at deployment time, and we calibrate $\alpha$ via grid search over $\{0, 0.1, \ldots, 1.0\}$ using a small labeled support set. We investigate how this calibration step depends on the size of the support set.

The landscape of $\alpha$ depends on the strength of the shortcut-induced shift. Under strong, controlled shifts where shortcut-label correlations from training are reversed (Table~\ref{tab:ood_results}), the grid search consistently selects large $\alpha$ to aggressively suppress shortcuts. Under in-distribution settings where these correlations still hold at test time (Table~\ref{tab:id_results}), shortcuts remain benign, it consistently selects small $\alpha$ (often $\alpha = 0$) to preserve the original signal. Real-world benchmarks fall in between: shortcut signals exist but are weaker and mixed with genuine predictive cues, so the optimal $\alpha$ is less clearly separated from neighboring values on the grid. Adaptive calibration is helpful when the appropriate debiasing strength is not known a priori.

Figure~\ref{fig:size_support} presents sensitivity analyses on real-world
datasets, comparing the adaptive variant against a fixed $\alpha = 1$ under
varying support set sizes. On SST-2, performance is stable across support
set sizes and coincides with the fixed $\alpha=1$ baseline, indicating that
the debiasing module can be applied at full strength here without loss. On
CivilComments, accuracy and WGA are slightly higher for $n \in \{40, 80,
160\}$ than for $n = 10$ and $20$. On MultiNLI, performance is considerably
more sensitive to the support set size. Overall, performance is stable for
$n \ge 40$ across all three datasets.

Overall, in our experiment we use support size $n=40$. The value of the calibration step is not in finding a globally optimal $\alpha$, but in adapting the LoRA strength to the unknown degree of distribution shift at deployment time, a flexibility that a single fixed $\alpha$ cannot provide.

\begin{figure*}
    \centering
    \includegraphics{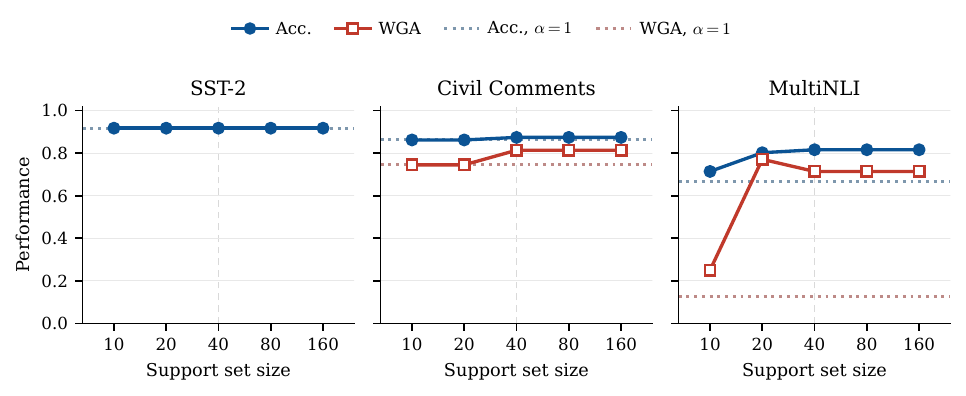} 
    \caption{Sensitivity to support set size. Test accuracy and
worst-group accuracy (WGA) of \textsc{Shortcut Guardrail} as a function of
the number of support examples used to calibrate the adaptive scalar
$\alpha$, on SST-2, Civil Comments, and MultiNLI.
Dotted lines mark the $\alpha\!=\!1$ baseline, which applies the debiasing module with no calibration. The datasets are the same as in Table~\ref{tab:realworld_results}.}
    \label{fig:size_support}
\end{figure*}


\subsection{\textsc{Shortcut Guardrail} Implementation Details}
\label{app:implementation}
We implement all models in PyTorch using the HuggingFace Transformers and PEFT libraries for the BERT encoder and LoRA adapters, and scikit-learn for the logistic-regression head in the DFR baseline; exact versions are provided in the released code. All results in Section~\ref{sec:finding} are reported as means with standard deviation over three random trials. Results in Section~\ref{sec:main} are from a single run per configuration.

\paragraph{Training Hyperparameters.}
Table~\ref{tab:training_hyper} lists the hyperparameters for training the LoRA-based debiasing module. We use a single set of values across all datasets, with the exception of LoRA rank which is scaled by dataset size (see below).

\begin{table*}[h]
\centering
\small
\begin{tabular}{lll}
\toprule
\textbf{Hyperparameter} & \textbf{Value} & \textbf{Rationale} \\
\midrule
\texttt{learning\_rate} & $1\text{e-}4$ & Standard for LoRA fine-tuning \\
\texttt{gradient\_accumulation\_steps} & 1 & Maximizes optimization steps; noisier gradients \\
 & & act as implicit regularization \\
\texttt{num\_train\_epochs} & 1 & Sufficient for ${\sim}$700--1000 samples \\
\texttt{temperature} & 0.1 & Standard InfoNCE temperature \\
\texttt{lora\_r} & 4 or 8 & Scaled by dataset size (see below) \\
\bottomrule
\end{tabular}
\caption{Training hyperparameters for the debiasing module.}
\label{tab:training_hyper}
\end{table*}

\paragraph{LoRA Rank Selection.}
We set LoRA rank proportional to dataset size: $r=4$ for datasets with approximately 700 samples (GoEmotions, CivilComments, MultiNLI, SST-2) and $r=8$ for datasets with approximately 1000 samples (Yelp). The adapter must have enough capacity to capture the contrastive signal but not so much that it overfits to noisy predictions. With ${\sim}$700 samples producing ${\sim}$6000-8000 contrastive pairs, $r=4$ provides a good bias-variance trade-off. Larger datasets generate more diverse pairs, so the additional capacity of $r=8$ can be utilized without overfitting. Empirically, Yelp at $r=4$ underperforms (WGA 0.146) compared to $r=8$ (WGA 0.264), while other datasets show stable or improved results at $r=4$.

\paragraph{Compute and model size.} Our backbone is BERT-base-uncased (approximately 110M parameters); the LoRA debiasing module adds a small number of trainable parameters (on the order of 147K for $r=4$, 295k for $r=8$). All experiments were run on a single NVIDIA A6000 GPU. Training the debiasing module for one epoch takes 5m23s minutes 5 minutes 23 seconds, as measured on the SST-2 dataset.

\subsection{Baseline Methods} \label{app:baselines}

\paragraph{Just Train Twice (JTT).}
We adopt the JTT method \cite{liu2021just_train_twice} using the authors' implementation on \href{https://github.com/anniesch/jtt}{GitHub}. JTT follows a two-stage training procedure without requiring group annotations. In Stage 1, a \texttt{BERT-base-uncased} model is fine-tuned on the shortcut-biased training set using standard empirical risk minimization (ERM). After training for $T_1$ epochs, the model's predictions are recorded, and misclassified examples are collected into an \textit{error set}. In Stage 2, a fresh model is trained on the full training set where error-set examples are upweighted by a factor $\lambda$. Following the original paper, we perform a grid search over learning rate $\in \{10^{-3}, 10^{-4}, 10^{-5}\}$, $T_1 \in \{1, 2\}$, and $\lambda \in \{4, 5, 6\}$, with Stage 2 fixed at $T_2 = 10$, batch size 16, and weight decay 0.01. The best configuration is selected by worst-group accuracy (WGA) on the validation set.

\paragraph{doNt Forget your Language (NFL).}
We use the NFL regularization method \cite{chew2024nfl} via the authors' implementation on \href{https://github.com/oscarchew/doNt-Forget-your-Language}{GitHub}. NFL fine-tunes the base model while penalizing deviations from the initial pretrained weights to preserve linguistic representations. We specifically use the \texttt{NFL-CP} variant, which applies an $\ell_2$ penalty on the squared difference between the fine-tuned and initial encoder weights across all \texttt{Linear} and \texttt{Embedding} layers. This penalty is scaled by $\lambda_r$ and added to the cross-entropy loss. We use the default hyperparameters provided by the authors and select the checkpoint with the highest validation accuracy.

\paragraph{Deep Feature Reweighting (DFR).}
We implement DFR \cite{kirichenko2023last_layer} as provided in the NFL codebase. DFR freezes a pretrained \texttt{BERT} model and the \texttt{[CLS]}-token representations are extracted, standardized, and used to fit an $\ell_1$-regularized logistic regression head on a randomly sampled small probe set of the training data (5\%). The regularization strength is selected via cross-validation over $\{1.0, 0.7, 0.3, 0.1, 0.07, 0.03, 0.01\}$ using a 50/50 split of the probe set.

\subsection{Positioning Against Existing Approaches}
\label{app:positioning}
Table~\ref{tab:positioning} situates \textsc{Shortcut Guardrail} (\textsc{SG}) within the landscape of shortcut-mitigation and test-time adaptation methods, comparing along three binary axes that characterize the assumptions a method makes, together with a summary of the mechanism it employs.
Comparison axes are explained below:
\begin{itemize}
    \item \textit{Deployment-time} indicates whether a method operates on a converged model using only test data, rather than intervening during training; consequently, it also entails no access to the original training set.
    \item \textit{No shortcut annotations} indicates whether a method avoids requiring explicit shortcut or group labels.
    \item \textit{Per-input attribution} indicates whether shortcut tokens are localized per example through gradient-based saliency, as opposed to being inferred from dataset-level statistics.
\end{itemize}
\textsc{SG} is the only method that satisfies all three axes: it operates at deployment time while using per-input attribution to localize shortcuts. Tent also operates at deployment time, but it targets general covariate shift via entropy minimization rather than spurious correlations and does not localize shortcuts, and is therefore not directly comparable on token-level shortcut benchmarks.
We focus on token-level shortcuts, the form studied by most prior work on spurious correlations in text classification, including LTGR~\citep{du2021towards}, C2L~\citep{choi2022c2l}, and NFL~\citep{chew2024nfl}.
\begin{table*}[ht]
\centering
\small
\setlength{\tabcolsep}{6pt}
\renewcommand{\arraystretch}{1.2}
\begin{tabular}{@{}lcccl@{}}
\toprule
Method & Deployment-time & \makecell{No shortcut\\annotation} & \makecell{Per-input\\attribution} & Approach \\
\midrule
Group DRO \citep{sagawa2020distributionally} & \xmark & \xmark & \xmark & Worst-group reweighting \\
JTT \citep{liu2021just_train_twice}          & \xmark & \cmark & \xmark & Error-based retraining \\
LTGR \citep{du2021towards}                   & \xmark & \cmark & \cmark & Attribution-guided training \\
DFR \citep{kirichenko2023last_layer}         & \xmark & \cmark & \xmark & Last-layer retraining \\
NFL \citep{chew2024nfl}                      & \xmark & \cmark & \xmark & Weight regularization \\
C2L \citep{choi2022c2l}                      & \xmark & \cmark & \cmark & Contrastive learning \\
Tent \citep{wang2021tent}                    & \cmark & \cmark & \xmark & Entropy minimization \\
\midrule
\textbf{\textsc{Shortcut Guardrail} (ours)}                  & \cmark & \cmark & \cmark & \makecell{Gradient-based attribution \\ + Masked contrastive learning} \\
\bottomrule
\end{tabular}
\caption{Positioning of \textsc{Shortcut Guardrail} against
existing shortcut-mitigation and test-time adaptation approaches.
\emph{Deployment-time} indicates operation on a converged model using only
test data, which entails no access to the training set.
\emph{Per-input attribution} indicates that shortcut tokens are localized
per example via gradient-based saliency rather than inferred from
dataset-level statistics.
\textsc{Shortcut Guardrail} is the only method that operates at deployment time while using
per-input attribution. Tent also operates at deployment time, but targets
general covariate shift via entropy minimization rather than spurious
correlations and does not localize shortcuts; it is therefore not directly
comparable on token-level shortcut benchmarks.}
\label{tab:positioning}
\end{table*}

\section{Proof of the Theory in Section~\ref{sec:theory}}\label{app:theory}

We prove that deployment-time shortcut mitigation faces a lower performance upper bound than training-time mitigation.

\paragraph{Preliminaries.}
Let $\mathcal{X,Y,V}$ denote the input space (e.g., text sequences), the label space (e.g., sentiment classification), and the vocabulary from which the tokens are drawn, respectively.
Let $\mathcal{D}_{\mathrm{train}} \in (\mathcal{X} \times \mathcal{Y})^N$ denote a training set of sample size $N$, and $\mathcal{D}_{\mathrm{test}} \in \mathcal{X}^M$ denote a test set of sample size $M$.
We define $S^* \subseteq \mathcal{V} \times \mathcal{Y}$ as the true spurious correlation structure yielded by the set of shortcut tokens (see the group definitions in Appendix~\ref{testbed}). We model $S^*$ as a random variable taking values in a finite set $\mathcal{S}$ with $|\mathcal{S}| \geq 2$.
We decompose shortcut mitigation into two stages: \textit{identification} and \textit{intervention}. 
In the \textit{identification} stage, we estimate the correlation structure $S^*$ by either training-time ($\hat{S}_{\mathrm{train}}$) or deployment-time methods ($\hat{S}_{\mathrm{deploy}}$).
In the \textit{intervention} stage, we adapt the model parameters $\theta$ given $\hat{S}_{\mathrm{train}}$ or $\hat{S}_{\mathrm{deploy}}$. 
We define $I(S^*;\hat{S})$ as the mutual information between the true and estimated correlation structure, and $H(S^*)$ as the entropy of $S^*$, reflecting the structure's complexity. 
We further define $R(\hat{S}, S^*)$ as a scalar-valued mitigation performance (e.g., overall accuracy or worst-group accuracy) of the method that produces the estimate $\hat{S}$, and $g(\theta, \mathcal{D}_{\mathrm{test}})$ as all possible deployment-time observables.

\begin{assumption}[Markov chain]\label{assump:markov}
    We assume that the following random variables form a Markov chain:
    \[
    S^* \;\to\; \mathcal{D}_{\mathrm{train}} \;\to\; \theta \;\to\; g(\theta, \mathcal{D}_{\mathrm{test}}).
    \]
    That is, conditioned on $\mathcal{D}_{\mathrm{train}}$, the model parameters $\theta$ are independent of $S^*$; and conditioned on $\theta$, the deployment-time observables $g(\theta, \mathcal{D}_{\mathrm{test}})$ are independent of $(\mathcal{D}_{\mathrm{train}}, S^*)$.
\end{assumption}

\thmInfoBound*
\begin{proof}[Proof of Theorem~\ref{thm:dpi}]
By Assumption~\ref{assump:markov}, $S^* \to \mathcal{D}_{\mathrm{train}} \to \theta \to g(\theta, \mathcal{X}_{\mathrm{test}})$ is a Markov chain. The Data Processing Inequality (DPI) states that for any Markov chain $A \to B \to C$, $I(A; C) \leq I(A; B)$. 
Applying the DPI to the first three and last three elements of the chain yields
\[
I(S^*;\, g(\theta, \mathcal{X}_{\mathrm{test}})) \;\leq\;I(S^*;\, \theta) \;\leq\; I(S^*;\, \mathcal{D}_{\mathrm{train}}).
\]
 
Since $\hat{S}_{\mathrm{deploy}}$ is a deterministic function of $g(\theta, \mathcal{X}_{\mathrm{test}})$, a further application of the DPI yields $I(S^*;\, \hat{S}_{\mathrm{deploy}}) \leq I(S^*;\, g(\theta, \mathcal{X}_{\mathrm{test}}))$. Combining these inequalities gives the result. Non-negativity of $\Delta I$ follows directly from the first inequality.
\end{proof}

\fanoErrBound*

\begin{proof}[Proof of Theorem~\ref{thm:fano}]
Fano's Inequality states that for any discrete random variable $X$ taking values in a set of size $K$ and any estimator $\hat{X}$ based on an observation $Z$:
\[
P(\hat{X} \neq X) \;\geq\; \frac{H(X) - I(X; Z) - 1}{\log(K - 1)}.
\]
 
\emph{Deployment-time bound~\eqref{eq:fano_deploy}:}\;
Since $\hat{S}_{\mathrm{deploy}}$ is a function of $g(\theta, \mathcal{X}_{\mathrm{test}})$, which is itself a function of $\theta$ (given the independent test inputs), we have $I(S^*;\, \hat{S}_{\mathrm{deploy}}) \leq I(S^*;\, \theta)$ by Theorem~\ref{thm:dpi}. Applying Fano's Inequality with $X = S^*$, $\hat{X} = \hat{S}_{\mathrm{deploy}}$, and $Z = \theta$ yields the result (noting that using $\theta$ in place of $g(\theta, \mathcal{X}_{\mathrm{test}})$ can only decrease the Fano bound, so the inequality remains valid).
 
\emph{Training-time bound~\eqref{eq:fano_train}:}\;
The training-time estimator $\hat{S}_{\mathrm{train}}$ has access to $\mathcal{D}_{\mathrm{train}}$. Applying Fano's Inequality with $X = S^*$, $\hat{X} = \hat{S}_{\mathrm{train}}$, and $Z = \mathcal{D}_{\mathrm{train}}$ yields the result.

\emph{Ordering:}\;
From Theorem~\ref{thm:dpi}, $I(S^*;\, \theta) \leq I(S^*;\, \mathcal{D}_{\mathrm{train}})$, so $L_{\mathrm{deploy}} \geq L_{\mathrm{train}}$ holds.
\end{proof}

\assumpmono*
\lemmaCeilingError*
\begin{proof}[Proof of Lemma~\ref{lem:perf_bound}]
Since $P(\hat{S} \neq S^*) \geq L$, we have $1 - P(\hat{S} \neq S^*) \leq 1 - L$. Because $\rho$ is non-decreasing (Assumption~\ref{assump:mono}):
\[
\mathbb{E}[R(\hat{S}, S^*)] = \rho\bigl(1 - P(\hat{S} \neq S^*)\bigr) \leq \rho(1 - L).
\]
Equations~\eqref{eq:perf_deploy} and~\eqref{eq:perf_train} follow by substituting $L = L_{\mathrm{deploy}}$ and $L = L_{\mathrm{train}}$ from Theorem~\ref{thm:fano}.
\end{proof}

\thmPerfBound*
\begin{proof}[Proof of Theorem~\ref{thm:main}]
\emph{Inequality:}\;
By Theorem~\ref{thm:fano}, $L_{\mathrm{deploy}} \geq L_{\mathrm{train}}$, which implies $1 - L_{\mathrm{deploy}} \leq 1 - L_{\mathrm{train}}$. Since $\rho$ is non-decreasing (Assumption~\ref{assump:mono}):
\[
\rho(1 - L_{\mathrm{deploy}}) \;\leq\; \rho(1 - L_{\mathrm{train}}).
\]
 
\emph{Gap characterisation:}\;
Since $\rho$ is differentiable by Assumption~\ref{assump:mono}, the Mean Value Theorem guarantees the existence of some $\xi \in [1 - L_{\mathrm{deploy}},\; 1 - L_{\mathrm{train}}]$ such that:
\[
\rho(1 - L_{\mathrm{train}}) - \rho(1 - L_{\mathrm{deploy}}) \;=\; \rho'(\xi) \cdot (L_{\mathrm{deploy}} - L_{\mathrm{train}}),
\]
where $\rho'(\xi)$ is the first-order derivative of $\rho$ at $\xi$.
Substituting the definitions~\eqref{eq:L_deploy} and~\eqref{eq:L_train}:
\begin{align*}
L_{\mathrm{deploy}} - L_{\mathrm{train}}
&= \frac{I(S^*;\, \mathcal{D}_{\mathrm{train}}) - I(S^*;\, \theta)}{\log(|\mathcal{S}| - 1)} \\[4pt]
&= \frac{\Delta I}{\log(|\mathcal{S}| - 1)}.
\end{align*}
Combining yields:
\[
\rho(1 - L_{\mathrm{train}}) - \rho(1 - L_{\mathrm{deploy}}) \;=\; \rho'(\xi) \cdot \frac{\Delta I}{\log(|\mathcal{S}| - 1)}.
\]
Since $\rho' \geq 0$ and $\Delta I \geq 0$ (Theorem~\ref{thm:dpi}), the gap is non-negative, confirming that the deployment-time performance ceiling is at most the training-time ceiling.
\end{proof}
\section{Extended Related Work}
\label{app:related_work}
 
\paragraph{Shortcut Learning.}
Shortcut learning refers to the phenomenon where deep neural networks rely on superficial statistical patterns rather than task-relevant features for prediction ~\citep{geirhos2020shortcut}. This issue is particularly prevalent in natural language understanding, where pretrained language models such as BERT ~\citep{devlin2019bert} have been shown to exploit spurious correlations present in training data ~\citep{du2022shortcut_survey, tu2020empirical}. A growing body of work has examined how these shortcuts manifest at the token level: specific words or phrases become disproportionately associated with target labels, causing the model to over-rely on their presence regardless of context ~\citep{wang2022identifying, du2023less_learn_shortcut}. For instance, in natural language inference, annotation artifacts such as negation words and lexical overlap between premise and hypothesis serve as well-documented shortcuts ~\citep{gururangan2018annotation, mccoy2019right}. Recent work has further expanded the scope of analysis from individual tokens to concept-level correlations ~\citep{zhou2024concept} and introduced controlled benchmarks with tunable shortcut strength for systematic evaluation ~\citep{zhou2024navigating}. While these studies provide valuable diagnostic insights, they primarily focus on characterizing the problem rather than addressing it under realistic deployment conditions.
 
\paragraph{Training-Time Shortcut Mitigation.}
A range of methods have been proposed to mitigate shortcut learning during model training, with progressively lighter supervision requirements. Group Distributionally Robust Optimization ~\citep[Group DRO;][]{sagawa2020distributionally} directly optimizes worst-group performance but requires explicit annotations of shortcut occurrence for every training example. To relax this requirement, Just Train Twice ~\citep[JTT;][]{liu2021just_train_twice} identifies shortcut-vulnerable samples via a proxy model and upweights them during retraining, eliminating the need for group labels while still requiring full access to the training data. Ensemble-based approaches ~\citep{clark2019dont_take_easy} train a bias-only model alongside the main model to downweight shortcut-reliant predictions, though they assume prior knowledge of the bias type. Deep Feature Reweighting ~\citep[DFR;][]{kirichenko2023last_layer} demonstrates that retraining only the last classification layer on a small group-balanced subset can recover robust performance, further reducing the retraining burden. More recently, NFL ~\citep{chew2024nfl} applies counterfactual data augmentation and regularization to weaken spurious correlations without group labels. Despite this progression toward lighter assumptions, all of these methods still require some combination of training data access, label supervision, or prior knowledge of the shortcut type. Moreover, shortcut behavior is inherently distribution-dependent: correlations that are benign or even helpful during training may become harmful only when the deployment-time distribution shifts. None of these methods can adapt to such shifts after training, when the model is fixed and ground-truth labels are unavailable. This is precisely the setting we address in this work.
 
\paragraph{Adapting Models Beyond Training.}
A separate line of work in computer vision has explored test-time adaptation (TTA), where model parameters are updated using unlabeled test data to handle distribution shifts. Tent ~\citep{wang2021tent} adapts batch normalization layers by minimizing prediction entropy at test time, and subsequent work has improved its efficiency and stability ~\citep{niu2022eata}. Other approaches adjust batch normalization statistics to the target distribution ~\citep{schneider2020improving, nado2020evaluating} or introduce self-supervised auxiliary objectives for test-time training ~\citep{sun2020test}. However, \citet{zhao2023pitfalls} identify several failure modes of existing TTA methods, including sensitivity to batch size and the severity of distribution shift. More fundamentally, these methods are designed for vision architectures and rely on properties that do not transfer to text: batch normalization layers (absent in transformer encoders like BERT), continuous pixel inputs (as opposed to discrete tokens), and spatial augmentations (which have no direct analogue in language). Furthermore, existing TTA methods target general distribution shift rather than the structured problem of token-level shortcut reliance. Our work addresses this gap by proposing a deployment-time adaptation framework specifically designed for mitigating token-level shortcuts in text classifiers, without requiring batch normalization or continuous input assumptions.
 
\paragraph{Gradient-Based Attribution and Contrastive Learning.}
Our method builds on two technical foundations. The first is gradient-based attribution, which estimates the contribution of each input feature to the model's prediction. Integrated Gradients ~\citep{sundararajan2017axiomatic} provides an axiomatic framework for this purpose, while \citet{li2016understanding} demonstrate that the simpler gradient $\times$ input method is effective for identifying influential tokens in NLP models. These techniques have been widely used for model interpretation, and more recently, \citet{du2022defi} connect attribution to debiasing by using interpretation scores to guide a teacher-student training procedure. However, prior work treats attribution primarily as an explanatory tool or a training-time signal; no existing method leverages it as an operational mechanism for deployment-time intervention. The second foundation is contrastive learning. SimCLR ~\citep{chen2020simclr} establishes the general framework of learning representations by contrasting augmented views of the same input, and \citet{khosla2020supervised} extend it to the supervised setting with improved robustness. In NLP, SimCSE ~\citep{gao2021simcse} adapts contrastive objectives to sentence encoders using dropout as minimal augmentation, and FairFil ~\citep{cheng2021fairfil} applies contrastive learning to debias pretrained text representations with respect to social attributes. More closely related to our setting, C2L ~\citep{choi2022c2l} combines contrastive learning with counterfactual augmentation to reduce reliance on spurious patterns, but still operates at training time with label supervision. Across these works, positive pairs are constructed through generic augmentations (dropout, random masking, synonym substitution) that are agnostic to which features the model actually relies on. Our method bridges these two lines by using gradient-based attribution to identify candidate shortcut tokens and constructing positive pairs through targeted masking of these tokens, so that the contrastive objective directly discourages shortcut reliance at deployment time.
\section{Ethics Considerations}
Our work addresses the robustness of text classifiers to spurious correlations, which is directly relevant to fairness in NLP systems. For example, in toxicity detection, spurious associations between demographic identity terms and the toxic label can lead to biased predictions against marginalized groups. Our method provides a practical tool for mitigating such biases at deployment time without requiring sensitive group annotations. All datasets used in this work are publicly available and have been widely used in prior research. We rely on the consent and ethical review processes of the original dataset creators. We do not collect or annotate any new data involving human subjects. 

All datasets used are established public benchmarks. CivilComments and GoEmotions are sourced from public online comments and were released by their creators with steps to remove or anonymize user identifiers; we use only these released versions and do not attempt re-identification. CivilComments and GoEmotions contain offensive or toxic language by design, as this is the prediction target; we did not augment or annotate this content, and we display only short, non-gratuitous excerpts where necessary for illustration.

\paragraph{Licenses.} The artifacts used in this work are publicly available and widely used in prior research. CivilComments~\citep{borkan2019nuanced} is released under the CC0 license, 
GoEmotions~\citep{demszky2020goemotions} and SST-2~\citep{wang2019glue} are under the Apache 2.0 license, 
and MultiNLI~\citep{williams2018multinli} is primarily under the Open American National Corpus (OANC) license, which allows all content to be freely used, modified, and shared under permissive terms; the Yelp review dataset~\citep{zhang2015character} and the Amazon Product Reviews~\citep{ni2019justifying} are distributed by their creators for research use. 
For models and training, we use the BERT encoder~\citep{devlin2019bert}, under the Apache 2.0 license, and LoRA~\citep{hu2022lora}, under the MIT license. Our use of all artifacts is consistent with their intended use for academic research.
We release our own code under the Apache 2.0 license.
\section{Potential Risks}
While our method reduces shortcut reliance, we acknowledge three potential risks. First, gradient-based attribution may mistakenly mask task-relevant tokens, suppressing genuine predictive signals. Second, adversarial actors aware of our token-level assumption could craft inputs exploiting multi-token shortcuts that our method fails to detect. Third, imperfect debiasing may create a false sense of security in high-stakes applications such as toxicity detection, where residual bias against marginalized groups may remain. We encourage careful evaluation before deployment.
\end{document}